\documentclass{article}

\usepackage[nonatbib,preprint]{neurips_2024}
\usepackage{graphicx} 
\usepackage{amsfonts}
\usepackage{amsmath}
\usepackage{amsthm}
\usepackage{tikz-cd}
\usepackage{authblk}

\usepackage[utf8]{inputenc} 
\usepackage[T1]{fontenc}    
\usepackage{hyperref}       
\usepackage{url}            
\usepackage{booktabs}       
\usepackage{amsfonts}       
\usepackage{nicefrac}       
\usepackage{microtype}      
\usepackage{xcolor}         
\newtheorem{theorem}{Theorem}
\newtheorem{corollary}{Corollary}[theorem]

\title{A rank decomposition for the topological classification of neural representations}
\author[1]{Kosio Beshkov}
\author[1,2,3]{Gaute T. Einevoll}
\affil[1]{Department of Biosciences, University of Oslo, Oslo, Norway}
\affil[2]{Department of Physics, University of Oslo, Oslo, Norway}
\affil[3]{Department of Physics, Norwegian University of Life Sciences, \AA s, Norway}
\date{}

\begin{document}

\maketitle

\begin{abstract}
    Neural networks can be thought of as applying a transformation to an input dataset. The way in which they change the topology of such a dataset often holds practical significance for many tasks, particularly those demanding non-homeomorphic mappings for optimal solutions, such as classification problems. In this work, we leverage the fact that neural networks are equivalent to piecewise-linear maps, whose local rank can be used to pinpoint regions in the input space that undergo non-homeomorphic transformations, leading to alterations in the topological structure of the input dataset. Our approach enables us to make use of the relative homology sequence, with which one can study the homology groups of the quotient of a space $\mathcal{M}$ and a subset $A$, assuming some minimal properties on these spaces. 
    
    As a proof of principle, we empirically investigate the presence of low-rank (topologically-destructive) affine maps as a function of network width and mean weight. We show that in randomly initialized narrow networks, there will be regions in which the homology groups of a data manifold can change. As the width increases, the homology groups of the input manifold become more likely to be preserved. We end this part of our work by constructing highly non-random wide networks that do not have this property and relating this non-random regime to Dale's principle, which is a defining characteristic of biological neural networks.
    Finally, we study simple feedforward networks trained on MNIST, as well as on toy classification and regression tasks, and show that networks manipulate the topology of data differently depending on the continuity of the task they are trained on.

\end{abstract}

\section{Introduction}

Imagine gazing out the window of a moving train. As it travels, we continuously take snapshots of the outside landscape. In some approximate sense, we can say that the visual structure of the world is being projected continuously onto our retina. Despite this, what captures our attention are the snapshots that are of particular interest to us, like an interesting building, an animal in the distance, or seeing the sun set over a large mountainous range. Given this duality in our perception, at some level, the neural networks of our brains have to hold both a continuous representation of the smoothly shifting scene as well as a discrete representation of the separate objects in each scene. Classifying on which side of this duality a particular neural network will fall into is yet to be understood.

To explore this question further, we have to formalize this intuitive picture. We can say that we take samples $\{x_1,x_2,...,x_N\}$ of some manifold $\mathcal{M}$ and then think of neural networks as applying a sequence of non-linear transformations to this manifold $\Phi:\mathcal{M} \to \mathbb{R}^N$. Depending on the architecture and the parameters of the neural network, this map can, loosely speaking, either be structure-preserving (homeomorphic) or structure-destroying (non-homeomorphic). Both types of mappings have their place in different types of tasks in both neuroscience and machine learning. Regression to a smooth function is an example in which employing a homeomorphic mapping seems like a quite sensible solution. Whereas classification tasks demand the removal of structure until only a few discrete points remain and therefore might benefit from a non-homeomorphic mapping.

The existence of these different types of solutions is hinted at by the structure of the activation functions we observe in systems implementing both biological as well as artificial intelligence. If one looks at the F-I curve of biological neurons \cite{izhikevich2007dynamical} or at the activation functions used in most modern artificial neural networks, it is clear that these functions send several inputs to the same point (usually 0). Therefore, they are not homeomorphic and specifically fail to be injective (one-to-one). As a result, one way in which neural networks can change the topological structure of an input manifold is by projecting several of its subsets to zero or to a subspace of lower dimension. To put this differently, a network can essentially glue different pieces of an input manifold together by projecting them to lower-dimensional subsets. Under this view, part of the question of how a particular input manifold can be transformed by an arbitrary neural network reduces to the question of what manifolds can be constructed by sequentially quotiening out submanifolds determined by the parameters of a network. This might not be an easy question in its own right, but it hopefully provides an interesting new vantage point from which to view neural networks.

The first step towards developing this type of understanding is to determine whether we can say something about which architectures and parameter regimes will lead to either a homeomorphic or non-homeomorphic mapping. To our knowledge, until now, this question has not been thoroughly explored. Let us first consider a network in which each layer is wider than the input dimension and all the weight matrices have a rank higher or equal to this dimension. Clearly, if we either choose a smooth activation function or take it away altogether, a neural network will apply a composition of smooth affine maps, which cannot change the topology of the input space. However, as we will show, by combining affine maps with a non-injective activation, both homeomorphic as well as non-homeomorphic mappings become possible.

We explore these questions by exploiting the idea of the polytope (often called polyhedral) decomposition of the input space of a neural network that is based in the field of hyperplane arrangements \cite{stanley2004introduction}, which by now has made a notable impact within machine learning \cite{huchette2023deep}. It has been used in the context of quantifying expressivity \cite{raghu2017expressive}, for estimating homology groups in combination with persistent homology \cite{liu2023relu}, and for many theoretical results \cite{pascanu2013number, rahaman2019spectral, hanin2019deep, rolnick2020reverse,masden2022algorithmic, grigsby2022functional, grigsby2022transversality}. Using this decomposition, it has been shown that the class of ReLU neural networks is exactly the same as the class of piecewise-linear functions with finite regions \cite{arora2016understanding, he2018relu, grigsby2022functional}. Therefore, one can express the action of a neural network as a piecewise linear map on the input space, through which we can further define another decomposition in terms of the rank of this map over the input. We show that topology can change exactly in regions with a rank smaller than the minimum embedding dimension of the input manifold. There is also the possibility that a piecewise linear map changes topology by cutting the input manifold along the boundaries of its pieces, we leave the study of this type of behavior for future work.

Following this, we introduce the idea of a \textit{representation classifying space} of an input manifold, which makes use of the relative homology groups of the input manifold with respect to subspaces that fall in the low-rank regions of the polytope decomposition. We analytically compute some examples with very simple manifolds. This is followed up by empirically studying randomly initialized narrow and wide neural networks, for which we provide conditions on the width and the average weight that determine whether such networks will implement homeomorphic or non-homeomorphic maps. We also explore the impact on input topology in networks following variations of Dale's principle. Finally, we examine which regime networks trained to solve regression and classification tasks converge to. We hope that this work provides an interesting perspective and opens up the path towards the application of more advanced topological methods in the study of neural networks.

\section{Studying the topology of neural representations}
\subsection{Neural networks and decompositions}
Given a manifold $\mathcal{M}$ embedded in $\mathbb{R}^m$, with $m$ being its minimal embedding dimension, and a neural network with L layers with corresponding widths $\{n_1,n_2,...,n_L\}$, we will represent a neural network by the following diagram:

\begin{equation}
    \begin{tikzcd}[row sep=small,column sep=small]
        \mathcal{M} \subset \mathbb{R}^m \arrow[r, "T_1"] & \mathbb{R}^{n_1} \arrow[r, "T_2"] & ... \arrow[r, "T_L"] & \mathbb{R}^{n_L}.
    \end{tikzcd}
\end{equation}

The maps $T_i: \mathbb{R}^{n_{i-1}} \to \mathbb{R}^{n_i}$ operate on an input $x$ by $T_i(x) = \text{max}\{0,W_ix+b_i\}$. In essence this mapping applies an affine transformation followed by a ReLU nonlinearity. We will denote the map which applies K of these maps in a sequence by $\Phi^K = T_K \circ T_{K-1} \circ ... \circ T_{1} = \prod\limits_{k=1}\limits^{K}T_k$ and refer to its image as the neural representation at layer $K$.

As mentioned above, it is known that ReLU networks decompose the input space into convex polytopes \cite{huchette2023deep, balestriero2019geometry}. This \textit{polytope decomposition} can be related to the neural representations through the codeword sets (where each codeword corresponds to a different polytope), 
\begin{equation}
C_J^K = \{x | S(\Phi^K(x)) = J\},\\
 S(x) =
\begin{cases}
    1 & \text{if } x > 0, \\
    0 & \text{if } x \leq 0.
\end{cases}
\end{equation}

In this equation $J$ is a binary vector denoting the neurons in layer $K$ which are active and $S(x)$ is a vector function. The quantity $||J||_0$ will be referred to as the \textit{codeword size}. These codewords form a decomposition of the input space at each layer of the network $\mathcal{M} = \bigsqcup\limits_{J}C_J^K$ for any $K$. We can also imagine a more global decomposition which is not layer dependent by defining $J$ over all the neurons in the network instead of layerwise. We call that the \textit{global polytope decomposition} $C_J$. Another feature of this decomposition is that within each codeword set, we can express the map $\Phi^K$ as an affine map $Q_{J_K}(W_K\Phi^{K-1}(x)+b_K)$, where $Q_{J_K}$ is a matrix which is 0 everywhere except on the diagonal $\text{diag}(Q_{J_K}) = J_K$. In a way it essentially substitues the ReLU function at layer $K$ for a piecewise linear map. As one can see this definition is recursive and by going backwards through the first layer, it can be written as,

\begin{equation}
    \label{eq: Phi rank}
    \Phi^K(x) = (\prod\limits_{k=1}\limits^{K} Q_{J_k}W_k)(x)+\sum\limits_{i=1}\limits^{K} \prod\limits_{j=i+1}\limits^{K} (Q_{J_j}W_j) Q_{J_i}b_i.
\end{equation}

Here the products are composed from the left and $Q_{J_k}$ is the diagonal $Q$ matrix at layer k. This is an affine transformation and if we want to study the rank of this map we can ignore the second term as it is simply an offset independent of the input. The rank tells us the dimension of the subspace to which a piece of the input manifold is projected to.\\

While this decomposition is useful and will become important later on, in order to study the topology of neural representations we find it more convenient to look at a slightly different decomposition which keeps track of the dimension of the manifold through the layers of the network, but ignores information about the exact codeword. We will call this the \textit{rank decomposition}, where we define the rank of a neural representation at a point $x$ as $\text{rank}(\Phi^K(x))=\text{rank}(\prod\limits_{k=1}\limits^{K} Q_{J_k}W_k)$. From this we obtain the rank codewords defined as,

\begin{equation}
    \Omega^K_n = \{x | \text{rank}(\Phi^K(x)) = n\}.
\end{equation}

These codewords form the rank decomposition $\mathcal{M} = \bigsqcup\limits_{n=0}\limits^{m}\Omega_n^K$ for any layer K. The advantage of this decomposition is that it allows us to split the input manifold into subsets which are mapped to a subspace of dimension less or equal to the input dimension. We find this to be a more useful perspective when studying when the topology of an input manifold changes in a neural representation. Figure \ref{fig:polytope decomp} shows an intuitive picture of how the polytope and rank decompositions interact with an embedded manifold. Further computational examples can be found in the appendix \ref{fig:other-manifolds}.

\begin{figure}
    \centering
    \includegraphics[scale=0.8]{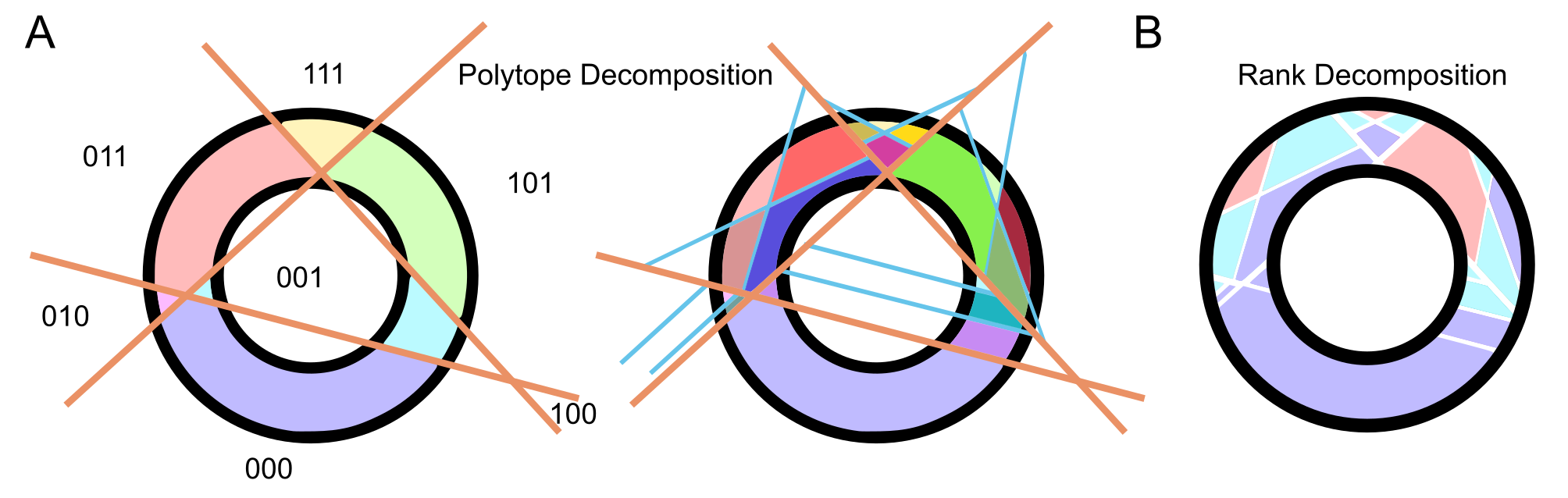}
    \caption{\textbf{A)}: Polytope decomposition at layers 1 and 2 of a neural network with a 2 layer architecture with widths of 3 and 2. \textbf{B)}: Rank decomposition of the same network at the second layer. Different colors correspond to different ranks: purple - rank 0, blue - rank 1 and pink - rank2.}
    \label{fig:polytope decomp}
\end{figure}

\subsection{Non-homeomorphic regions of neural manifolds}
Since full-rank affine maps between vector spaces are continuous (in both directions) and invertible bijections, they induce a homeomorphism between manifolds (equipped with the subspace topology in $\mathbb{R}^n$) embedded in these spaces and are therefore not able to change the topology of a manifold. In the piecewise linear case, this reasoning can be applied within each region. Then one way the topology can change is when such an affine map projects subsets of the manifold to a subspace of a lower dimension than its embedding dimension. This is due to the fact that the projected subsets can no longer be bijectively embedded in this new lower-dimensional space. Such a degenerate map is non-injective and therefore non-homeomorphic. For a full proof see the appendix \ref{Topology destructive regimes}.

In order to classify the types of manifolds that one can get through such mappings, we focus on studying a standard topological invariant, namely homology groups \cite{hatcher2005algebraic}. Assuming one is working with nice (compact and Hausdorff spaces as shown in theorem 2.14 and its corollary 2.15 in \cite{vick2012homology}) input spaces, when the map $\Phi$ between manifolds fails to be injective, studying the homology groups $H(\Phi(\mathcal{M}))$ is equivalent to studying $H(\mathcal{M} / \sim_\Phi)$, where $\sim_\Phi$ contains the equivalence classes induced by the equivalence relation $\Phi(x)=\Phi(y)$ and $\mathcal{M} / \sim_\Phi$ is equipped with the quotient topology.

Since the mapping $\Phi$ is a homeomorphism on all regions $\Omega_n$ of rank equal to $m$, we will only consider what happens to regions of lower rank. Let $\Omega_0^K$ be a rank 0 subset from the input manifold which is projected to a single point. If this is a contractible space, projecting it to a point will not change any of the homology groups. Otherwise, quotiening this space out can have an impact. Both of these cases can be studied through the relative homology sequence \cite{hatcher2005algebraic}, which is defined for a pair of spaces $(X,A)$ with $A \subset X$. If one has a map which is homeomorphic on all points of $\mathcal{M}$ except on $\Omega_0^K$ which is all projected to a point, then the relative homology sequence of the pair $(\mathcal{M}, \Omega_n^K)$ can be studied. Since $\Phi$ is not necessarily such a map, we construct an intermediate map which we define by,

\begin{equation}
    \Psi^K_0(x) = 
    \begin{cases}
        [p] \text{ if } x \in \Omega^K_0, \\
        \text{id otherwise.}  
    \end{cases}
\end{equation}

Here $[p]$ is the equivalence class of all points inside $\Omega^K_0$. A more visual way to think of this map is that it keeps everything the same but glues together all points of $\Omega^K_0$. In other words it maps $\mathcal{M} \to \mathcal{M} / \Omega^K_0 = \Psi_0^K(\mathcal{M})$, which is the precise context in which the relative homology sequence can be applied and is visualized by the diagram,

\begin{equation}
    \begin{tikzcd}[row sep=small,column sep=small]
        ... \arrow[r, "\partial^*"] & H_i(\Omega_0^K) \arrow[r, "\iota^*"] & H_i(\mathcal{M}) \arrow[r, "j^*"] & H_i(\mathcal{M},\Omega_0^K) \arrow[r, "\partial^*"] & H_{i-1}(\Omega_0^K) \arrow[r, "\iota^*"] & ...
    \end{tikzcd}
\end{equation}

Since this sequence is exact, we know that the image of the map at each arrow is equal to the kernel of the map at the next arrow. This property will let us analytically calculate and classify the homology groups of different manifolds as they are propagated through the layers of a network.

\subsubsection*{Contractible and non-contractible $\Omega_0^K$ have different effects on homology groups}
By using the relative homology sequence one can easily show that if $\Omega_0^K$ is contractible, the homology groups of the intermediate map will be the same as the input manifold. Start by using the fact that contractible spaces have zero homology in non-zero degrees. This means that we can fill in the relative homology sequence for $i>1$ in the following way,

\begin{equation}
    \begin{tikzcd}[row sep=small,column sep=small]
        ... \arrow[r, "\partial^*"] & 0 \arrow[r, "\iota^*"] & H_i(\mathcal{M}) \arrow[r, "j^*"] & H_i(\mathcal{M},\Omega_0^K) \arrow[r, "\partial^*"] & 0 \arrow[r, "\iota^*"] & ...
    \end{tikzcd}
\end{equation}

By exactness we can see that $\text{Ker}\partial^*=\text{Im}j^*$ and since the image of $\partial^*$ is empty this is the only set that determines $H_i(\mathcal{M}, \Omega_0^K)$. On the other hand, for non-contractible spaces the image of $\partial^*$ or $\iota^*$ might fail to be empty and can therefore induce a change in $H_i(\mathcal{M}, \Omega^K_0)$. This precise change can be studied through the relative homology sequence in which $H_{i}(\Omega^K_0)$ or $H_{i-1}(\Omega^K_0)$ are non-trivial.

\begin{equation}
    \begin{tikzcd}[row sep=small,column sep=small]
        ... \arrow[r, "\partial^*"] & H_{i}(\Omega^K_0) \arrow[r, "\iota^*"] & H_i(\mathcal{M}) \arrow[r, "j^*"] & H_i(\mathcal{M},\Omega_0^K) \arrow[r, "\partial^*"] & H_{i-1}(\Omega^K_0) \arrow[r, "\iota^*"] & ...
    \end{tikzcd}
\end{equation}

The 0 rank case is quite nice to work with since there is only one unique projection from $\Omega_0^K$ to a point. Higher rank regions $\Omega^K_n$ are projected to an $n$-dimensional linear subspace, and therefore lack this uniqueness property. For rank 1 regions we can still get quite far, since the only difference is that pieces of the input manifold can now be sent to disconnected (but contractible) subsets of a line. In other words, $\Omega_1^K = \bigsqcup V_j$, where each $V_j$ is a different disconnected component of $\Omega_1^K$. Any disconnected component is contractible when projected to a line. Therefore we can construct a set of intermediate maps and apply the same tricks as in the zero dimensional case in a sequence.

\begin{equation}
    \begin{tikzcd}[row sep=small,column sep=small]
    ... \arrow[r, "\partial^*"] & H_{i}(\Omega^K_0) \arrow[r, "\iota^*"] & H_i(\mathcal{M}) \arrow[r, red, "j^*"] & H_i(\mathcal{M},\Omega_0^K) \arrow[r, "\partial^*"] \arrow[ld, red, "id"] & ...\\
    ... \arrow[r, "\partial^*"] & H_{i}(\Omega^K_{(1,0)}) \arrow[r, "\iota^*"] & H_i(\mathcal{M},\Omega^K_0) \arrow[r, red, "j^*"] & H_i(\mathcal{M}/\Omega_0^K,\Omega_{(1,0)}^K) \arrow[r, "\partial^*"] \arrow[ld, red, "id"] & ...\\
    ... \arrow[r, "\partial^*"] & H_{i}(\Omega^K_{(1,1)}) \arrow[r, "\iota^*"] & H_i(\mathcal{M}/\Omega_0^K,\Omega_{(1,0)}^K) \arrow[r, red,  "j^*"] & H_i((\mathcal{M}/\Omega_0^K)/\Omega_{(1,0)}^K,\Omega_{(1,1)}^K) \arrow[r, "\partial^*"] & ...
    \end{tikzcd}
\end{equation}

Looking at the red arrows, it is clear that we are studying the sequence induced by the inclusions of each contractible low rank region $\Omega_0^K \subset \Omega_0^K\cup\Omega_{(1,0)}^K \subset \Omega_0^K\cup\Omega_{(1,0)}^K\cup\Omega_{(1,1)}^K \subset ...$. Intuitively this is the same as sequentially gluing different disjoint regions of $\mathcal{M}$.

It is tempting to keep going with the same approach for higher rank regions, but the problem one runs into is that the projections need not be contractible and therefore the concept of quotiening $\Omega_{n>1}$ does not always apply. For example, the two dimensional projections of a cylinder, can look like a circle, or a contractible space depending on the angle of projection. Solving this problem for arbitrary $n$ seems quite difficult and is out of the scope of the present work. However, if we focus on classifying the cases for which $\Phi(\Omega_n^K)$ maps to a finite set of individually contractible sets, we can get a good idea for just how rich such transformations can be. We study this further and provide a conjecture for how to define such \textit{representation classifying spaces} in the appendix \ref{sec: appendix}.

\section{Combinatorial perspective on codeword domains}
So far we have studied how neural networks can change the topological structure of an input manifold, without any reference to the exact architecture of a network. This approach, while quite general, does not help answer the practical question of how particular architectural choices impact the topology of neural representations. While there are $2^{n_K}$, possible binary codewords that one can construct it is well known from previous work that the number of linear regions at a layer is bounded by a much smaller quantity \cite{pascanu2013number, montufar2014number, arora2016understanding, raghu2017expressive, serra2018bounding}. The tightest upper bound of this number, that we are aware of, is given by Serra et al. \cite{serra2018bounding}. 

Let us naively assume that each codeword region $C^K_J$ has an equal probability of appearing in the polytope decomposition of a neural network initialized from a random distribution. As argued above, the regions capable of changing the topology of an input manifold, which we will call \textit{topologically destructive}, are those of rank lower than the minimal embedding dimension $m$. In principle, there are $\sum\limits_{k=0}\limits^{m-1} \binom{n_K}{k}$ out of $2^{n_K} = \sum\limits_{k=0}\limits^{n_K} \binom{n_K}{k}$ such regions, with the remainder being \textit{topologically preserving}. Since the number of regions is restricted by the bound of Serra et al., in practice it might only be possible to fit a limited number of regions, none of which need to be destructive. Also, since the binomial coefficient for a fixed width $n_K$ is maximized at $k=\lfloor\frac{n_K}{2}\rfloor$ and $k=\lceil\frac{n_K}{2}\rceil$, we would expect that most regions would be near these ranks.



Using this reasoning, note that the number of destructive regions will rapidly decrease when $n_K>>m$. Therefore, as the width increases relative to the input dimension, the probability of having a topologically destructive region goes to 0. In the following section we empirically study this relationship between layer width and manifold dimension in single layer networks.

\subsection{Narrow and wide networks have a different impact on manifold topology}

To observe the effects of network width on the existence of topologically destructive regions, we ran several single layer simulations, in which we varied the sizes of the input and output layers. We choose both the input $m$ and output $n_1$ dimensions to be \{2,5,10,25,50,100\} and performed 100 different Kaiming initializations \cite{he2015delving} for each input-output dimension combination. We uniformly sampled 1000 points from the $m$-dimensional hypercube $[-5,5]^m$ and computed the rank at each point following equation (\ref{eq: Phi rank}). To confirm whether there is a topologically destructive region, we plot the minimum rank over all such points.

As can be seen in panel A of Figure \ref{fig:Top destruction}, as the width of the output layer increases the minimal rank increases until it reaches the size of the input dimension. This convergence is slower for larger input sizes and for the cases in which $m \geq n_1$ it is guaranteed that there will be many topologically destructive regions. This means that the degree to which a random neural network contains low rank topologically destructive regions, depends on the relation between the width of the layers and the minimal embedding dimension of a manifold. 
In addition, we ran simulations in which we increased the mean of the weight distribution and observed that adding this bias to the mean severely lowered the minimal rank region of the networks, implying they became more topologically destructive. These results can be seen in panel B of Figure \ref{fig:Top destruction}.

\begin{figure}
    \centering
    \includegraphics[scale=0.85]{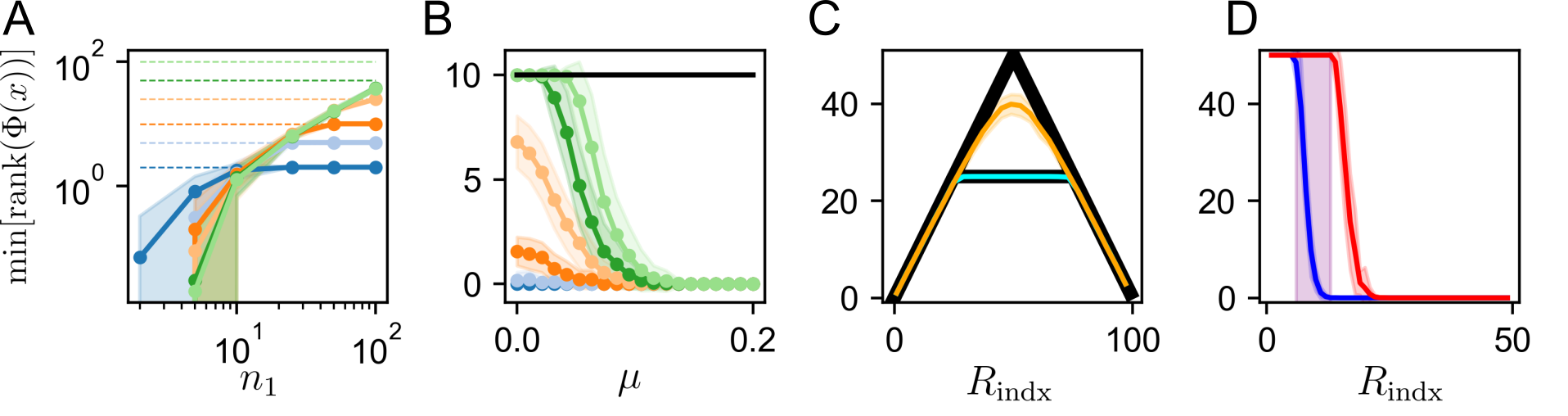}
    \caption{\textbf{Topological destruction in random neural networks} \\
    \textbf{A)} The average minimum rank region observed in 100 neural networks with Kaiming initialization increases with the width of the output ($n_1$) and is bounded by the input ($m$) layer width. Colors encode the size of the input layer with blue (m=2), light blue (m=5) orange (m=10), light orange (m=25), green (m=50), and light green(m=100). Networks with ranks below the punctured lines, matched to the colors of the input dimension, contain topologically destructive regions. 
    \textbf{B)} The minimal rank decreases as a function of the mean $\mu$ of the weight distribution (normal with $\sigma=0.1$). The input dimension here is fixed at $m=10$ (black line). The colors denote networks with different output widths $n_1$ following the same sequence as in panel \textbf{A)}.
    \textbf{C)} The minimal rank region of a neural network with 25  (cyan) or 50 (orange) input neurons with strictly positive weights multiplied by a rank selection matrix scales sublinearly as a function of the index of $R$. The two black curves show the bounding functions $\min\{R_{\text{indx}}, n_1-R_{\text{indx},m}\}$ for the two corresponding networks, with $n_1=100$.
    \textbf{D)} The minimal (blue) and maximal (red) ranks of a balanced connectivity (inhibitory connections are 4 times as strong as excitatory ones) network following Dale's principle with 50 input and 100 output neurons, evaluated over strictly positive samples. The region highlighted in purple represents the regime in which there are both topologically destructive as well as topologically preserving regions.}
    \label{fig:Top destruction}
\end{figure}

\subsection{Construction for topology destroying wide networks}
What this points towards is that narrow networks are more likely to change the topology of an input manifold than wide networks. However, we also saw that by increasing the mean of the weight distribution, we get networks that contain regions of lower rank (and are therefore more topologically destructive). Thus, we know that by choosing the right weight distribution, we can construct wide neural networks with topologically destructive regions. One simple example is sampling the weights from the uniform distribution on the interval $[-1,0]$ and setting the bias to 0. Such a network guarantees that there will be a 0 rank region in $\mathbb{R}^m$, since any point with strictly positive values will fall in the non-injective region of the ReLU activation.\\


This guarantees the existence of rank 0 regions, and we can run the same argument to construct regions of a specific rank. First let us define a rank selection matrix,

\begin{equation}
    R_{i,j}(k,n_1) = 
    \begin{cases}
        1 \text{ if } i=j \text{ and } i\leq k,\\
        -1 \text{ if } i=j \text{ and } k < i \leq n_1, \\
        0 \text{ otherwise}.
    \end{cases}
\end{equation}

If we multiply the strictly negative weights by this matrix from the left $R_{i,j}(k,n_1)W$, we guarantee that there will be $k$ neurons which will be active at any strictly positive point, therefore the rank of this region will now be $k$. However, that also implies that if we flip the sign of such a point, there will be $n_1-k$ active neurons. Thus, starting with strictly positive weights, the minimal rank of such a network will be upper bounded by $\min\{k,n_1-k,m\}$. We verify this bound in panel C of Figure \ref{fig:Top destruction}. The more neuroscientifically inclined reader, will notice that this condition is the equivalent of imposing a modified Dale's principle, in which the incoming, rather than the outgoing, connections to a neuron have the same sign.

While we did this construction for a single layer, the generalization to deeper networks is quite straightforward. To see this, note that $\Phi^1$ maps any input to $\mathbb{R}^{{n_1}^+}$, so choosing the weights in the next layer from a distribution on $[-1,0]$ will similarly lead to a reduction in the rank. Furthermore, deeper layers in artificial and biological neural networks will only have positive outputs, which means that increasing inhibition will usually lead to lower ranks and more topological destruction. 

To study this more biologically relevant case, we also ran simulations in which we looked at the ranks produced by networks following Dale's principle (which is equivalent to multiplying $W$ from the right or $WR_{i,j}(k,m)$). As can be seen from panel D in Figure \ref{fig:Top destruction}, strongly excitatory networks (before the purple region) have no topologically destructive regions. Balanced, meaning there is equal excitation and inhibition \cite{van1996chaos, deneve2016efficient}, networks (in the purple region) have both topologically preserving as well as destructive regions. Finally, strongly inhibitory networks (after the purple region) only have topologically destructive regions.

\section{How topologically destructive are neural networks after training?}

So far we have focused only on studying randomly initialized networks, but the real power of neural networks comes from our ability to train them to perform complex tasks. Therefore, in this section we study what happens to the ranks of the maps in trained neural networks. We start off by training a simple feedforward neural network on MNIST and plot the histogram of the ranks on the test data. The dimensionality of the MNIST manifold is estimated to be between 7 and 13 \cite{pope2021intrinsic}. As a result we cannot determine the precise integer value at which topology is destroyed. However, we can still observe whether after training, the ranks around the test samples are lower than before that. To test this, we randomly initialized a feedforward neural network with layers of widths \{2000, 1000, 500, 100\} and trained for 50 epochs with a batch size of 64, using the Adam optimizer \cite{kingma2014adam} with default parameters and a cross-entropy loss. We achieved a testing accuracy of $> 98.4\%$.

As can be seen in panels A and B of Figure \ref{fig:training_ranks}, before training, both MNIST and random noise samples end up in regions of average rank around the layer width divided by two. However, during and after training the MNIST samples end up in regions of much lower rank, with an average close to the lower bound of the dimension estimate for MNIST. A similar, but much less extreme reduction in rank can be seen in the regions to which random samples are projected after training. Thus we can think of the classification training process as a search for a polytope decomposition which implements a more topologically destructive map. 

The exploration for more topologically destructive maps makes sense in the context of classification problems in which the goal is to disentangle the input data into discrete classes \cite{naitzat2020topology}. However, when it comes to regression tasks, this intuition need not apply. If anything it would make sense to avoid badly behaved maps. To empirically confirm this we trained neural networks to approximate discretizations of the sphere function $\sum\limits_{i=1}\limits^{d}x_i^2$ with $d=20$. This was done by defining the n-th discretization as $f_n(x) = k/n$ where $P_{k/n}\leq f(x) <P_{(k+1)/n}$ and $P_k$ is the k-th percentile of the function. We uniformly sampled 10000 points from the interval $[-5.12,5.12]^{20}$ and generated the corresponding z-scored $\{2,3,4,5,8,\infty\}$-th discretizations. We trained 10 networks with layers of width \{160,80,40\} per discretization. A batch size of 64 was used along with the Adam optimizer with a learning rate of 0.001. The mean squared error (MSE) was used as the loss.

For better visibility, we computed the codeword size instead of the rank. With this measure any codeword size under 20 (which is the dimension of the input) is guaranteed to be topologically destructive, whereas entries higher or equal to that are not. As can be seen in panel C of Figure \ref{fig:training_ranks}, the lower discretizations, $f_2$ and $f_3$ specifically, which are more similar to classification tasks, on average had regions of lower codeword size compared to the pre-learning baseline. However, when the continuous regression version $f_\infty$ of this task is reached the codeword size of the regions was actually increased on average. This result supports the idea that neural networks solve classification tasks by orienting samples in topologically destructive regions. Regression tasks, on the other hand, are better solved in a higher rank regime, where smoothness is desired. We also explore the rank distributions of several other standard high-dimensional optimization functions taken from  \url{http://www.sfu.ca/~ssurjano} \cite{surj2013} in the appendix \ref{Other regressions}.

\begin{figure}
    \centering
    \includegraphics[scale=0.85]{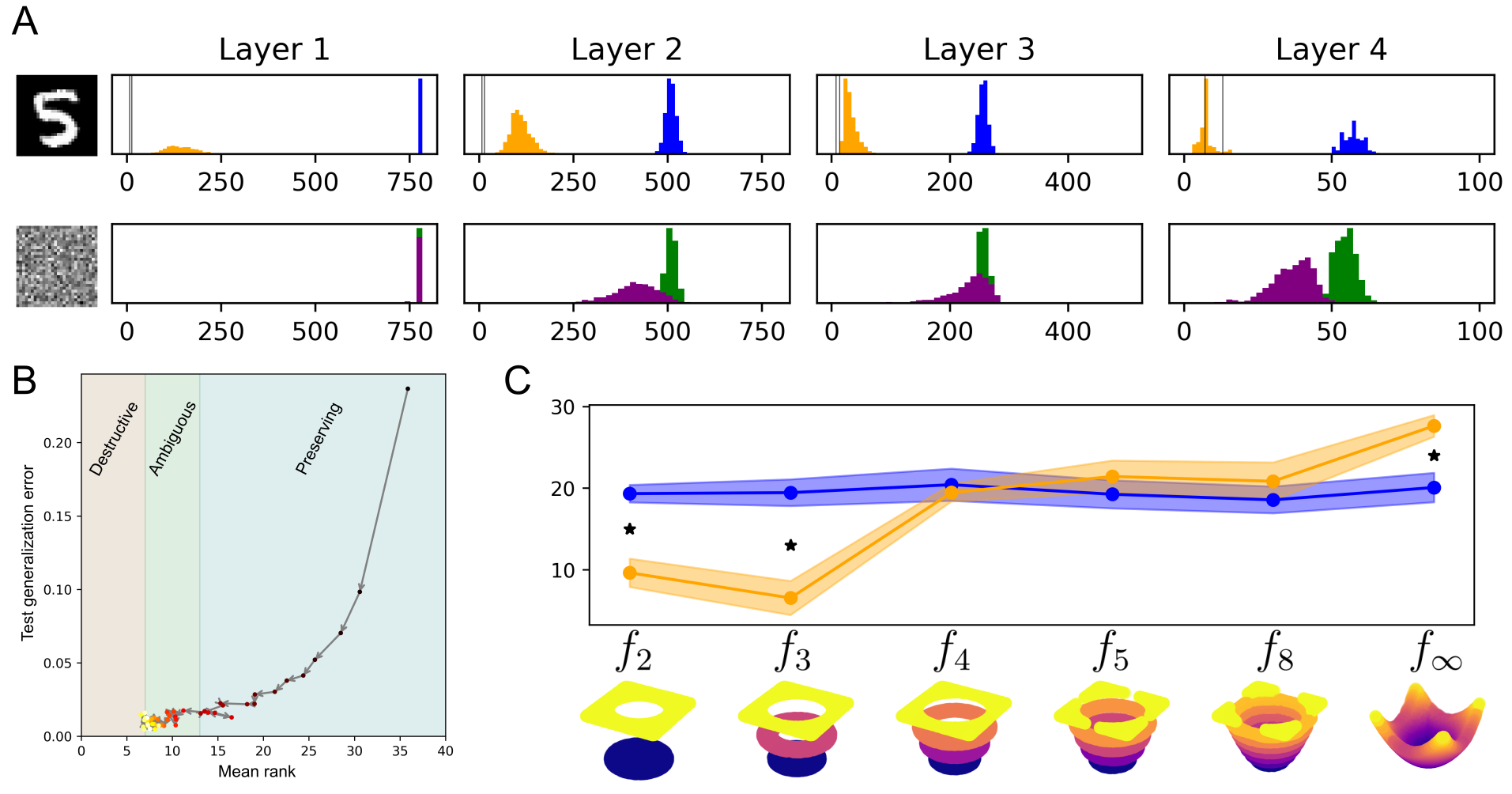}
    \caption{\textbf{Trained neural networks partition the input space in different ways for classification and regression problems.}\\
    \textbf{A}) The top row shows histograms of the ranks across layers for MNIST samples before (blue) and after (orange) training. The two vertical black lines show the intrinsic dimensionality of the MNIST dataset estimated by \cite{pope2021intrinsic}. The row beneath shows the same, but for normally distributed noise samples instead of MNIST digits before (green) and after (purple) training.
    \textbf{B}) A drop in generalization, quantified in terms of the cross-entropy loss on the test data, corresponds to a reduction in the rank of the average region in layer 4. Since the true dimensionality of MNIST is not known, we highlight the region where topological destruction might start (7-13) according to the dimension estimate in \cite{pope2021intrinsic}. As training proceeds the mean rank converges towards the lower bound of this estimate, indicating that in this case optimal performance is achieved by identifying topologically destructive regions.
    \textbf{C)}  The average codeword size increases as a function of the number of values in the output. Plots of two dimensional functions discretized at different levels of resolution. Stars indicate a Bonferroni corrected significant difference (at $p < 0.001$) between post (orange) and pretraining (blue) distributions using a Wilcoxon rank-sum test.}
    \label{fig:training_ranks}
\end{figure}

\section{Discussion}
Ideas from geometry and topology have undergone an increase in popularity in both machine learning and neuroscience. In this work we have attempted to build a bridge between the structure of a neural network and the way in which it can change the topology of an input manifold. For this purpose we have made use of the relative homology sequence, which is a fundamental tool in algebraic topology \cite{hatcher2005algebraic}, but to our knowledge, with the exception of a few works \cite{pokorny2016topological, blaser2022relative}, has not received as much attention in machine learning, neuroscience and data analysis. By constructing the rank decomposition of a neural network and thereby separating the input manifold into subregions, we are able to identify those regions where topological changes could occur. Then the problem of how the topology of the input changes, reduces to the problem of how these regions are projected to hyperplanes of different dimension lower than the input dimension.

Our theory is easily applicable to regions of rank zero and one. However, higher rank regions pose a problem, in that projections of such subregions, need not be contractible and therefore the relative homology sequence is not straightforwardly applicable. Future work, can attempt to shine light on how one should approach these types of projections. By focusing on the contractible case, we have conjectured that this classification can be described by considering the quotient spaces generated by submanifolds of the original space. We are hopeful that these ideas will be explored further in future work. 

We also evaluated the degree of topological destruction in random neural networks and showed that narrow networks lead to topological destruction whereas wide ones preserve the topology. This principle can be broken, either by biasing the mean of the weight distribution or by making the rows of the weight matrix to be of consistent sign. Assuming that the input has already passed through one ReLU layer, this sign-consistency condition reduces to having only negative weights.

The minimum rank region of a neural network can be manipulated by changing the sign of the rows of the weight matrix. On the other hand changing the sign of the columns, preserves Dale's principle and in a balanced regime, guarantees that there will be both topologically destructive as well as preserving regions. Such connectivity patterns, have also been associated with optimal distance-preserving low-dimensional projections \cite{allen2014sparse}. We believe that another fascinating future avenue of research would be to explore balanced rate-based recurrent neural networks obeying Dale's principle. Such networks, could potentially provide us with insight about how neural representations change through time and thus could correspond to different stages of processing during the solution of a task. 

Finally, we study the appearance of topologically destructive regions during neural network training. We empirically observe that a decrease in generalization error is assosiated with a corresponding decrease in the average rank of a network trained on MNIST. This behavior is also consistently observed on other classification tasks, whereas the opposite seems to be true for regression tasks. Future work could focus on studying this type of behavior in more complex tasks and network architectures. 

\section{Acknowledgements}
We wish to thank Jonas Verhellen and Krubeal Danieli for their insightful input regarding the figures and the content of this manuscript. This research was funded by the European Union’s Horizon 2020 research and innovation programme under the Marie Skłodowska-Curie grant agreement Nº 945371 and the University of Oslo.

\bibliographystyle{plain}
\bibliography{main}

\newpage

\section{Appendix}
\label{sec: appendix}

\subsection{Code availability}
All of our code is openly available at \url{https://github.com/KBeshkov/top-mfld-class/tree/main}.

\subsection{Topology changes in low rank regions}
\label{Topology destructive regimes}

\begin{theorem}
    \label{Theorem 1}
    Given a manifold $\mathcal{M} \subset \mathbb{R}^m$, with a minimal embedding dimension of $m$ and a linear transformation $T:\mathbb{R}^m \to \mathbb{R}^n$, $T|_{\mathcal{M}}:\mathcal{M} \to T(\mathcal{M})$ is a homeomorphism iff $\text{rank}(T) \geq m$. 
\end{theorem} 

\begin{proof}
    For $T|_{\mathcal{M}}$ to be a homeomorphism, it has to be bijective, continuous and its (pseudo)inverse $T^{-1}$ also has to be continuous.  Since $T$ and $T^{-1}$ are both linear transformations between finite dimensional topological vector spaces and all linear transformations between finite dimensional topological vector spaces are continuous, continuity follows in both directions. $T|_{\mathcal{M}}$ is surjective by construction, therefore the only property that we have to prove is injectivity. Using the rank-nulity theorem, we can write:

    \begin{equation*}
        \text{rank}(T|_{\mathcal{M}}) + \dim(\ker T|_{\mathcal{M}}) = \dim(\mathcal{M}).
    \end{equation*}
    Injectivity follows only when $\dim(\ker T|_{\mathcal{M}})=0$. If we assume that $T$ has rank $n$, then rearranging the terms of the equation above we get,
    
    \begin{equation*}
        \dim(\ker T|_{\mathcal{M}}) = m-n.
    \end{equation*}
    Which shows that as long as $\text{rank}T|_{\mathcal{M}} = n$ is smaller than $\dim \mathcal{M}$, the kernel will be fail to be empty and the map will not be a homeomorphism. If it is equal to $m$, then the kernel is empty and the map is injective and therefore homeomorphic. \\
\end{proof}

\begin{theorem}
    \label{Theorem 2}
    Take a neural network with a rank decomposition $\Omega$ which implements a map $\Phi$ and an input manifold $\mathcal{M} \subset \mathbb{R}^m$, with $m$ being the minimal embedding dimension of the manifold. If $\Omega_n \cap \mathcal{M} = \mathcal{M}$ for $n<m$, then $\Phi$ is non-homeomorphic.
\end{theorem}

\begin{proof}
    $\Phi$ applies a set of rank $n$ affine maps on $\mathcal{M}$ which project subsets of $\mathcal{M}$ to $\mathbb{R}^n$. There is no embedding of $\mathcal{M}$ into $\mathbb{R}^n$, therefore $\Phi$ cannot be a homeomorphism.
\end{proof}

\begin{corollary}
    \label{Corollary 1}
    $\Phi$ is non-homeomorphic on any submanifold of $\mathcal{M}$ with the same embedding dimension, which intersects $\Omega_n$ for $n<m$.
\end{corollary}

\begin{proof}
    Simply apply Theorem \ref{Theorem 2} to such submanifolds of $\mathcal{M}$.
\end{proof}

\subsection{Conjecture about higher order regions}
Targeting the problem of classifying the neural representations for high-dimensional non-trivial projections is quite difficult. Let us instead focus on classifying the cases for which the topologically destructive regions of $\Phi(\Omega^K_n)$ map to a finite set of individually contractible sets. In this case we only have to study how quotiening combinations and sequences of subspaces of $\mathcal{M}$ changes the topology of the original manifold. For this purpose, we first define the set of all possible non-homotopic subsets of $\mathcal{M}$ as $Z_\mathcal{M}$. We know that if $(X,\Omega_n)$ and $(X,\Omega_n^{'})$ are homotopy equivalent as pairs, then $H_i(X,\Omega_n) \simeq H_i(X, \Omega_n^{'})$ for any i. On the other hand, quotiening out non-homotopic subsets can lead to differences in relative homology (think of quotiening out a trivial loop vs the middle loop in a torus). \\

In general this leads to a very complicated classification problem, in which one has to consider quotiening out arbitrary non-homeomorphic subspaces of $\mathcal{M}$ along with all non-homotopic mappings of such subspaces. Not only are there just too many such subspaces, especially as one considers higher dimensional cases, but the fact that non-homotopic maps can generate different topologies makes things extremely intractable. 

In order to frame the problem in better terms and conjecture a solution, we name the set of all possible spaces which a neural network can generate, \textit{the representation classifying space},
\begin{equation}
    [W]_\mathcal{M} = \{H(\mathcal{M},z) | z \subset Z_\mathcal{M}\}.
\end{equation}

\textbf{Conjecture:} Given a manifold $\mathcal{M}$, let us call the equivalence class of all possible maps from submanifolds into $\mathcal{M}$, 
\begin{equation*}
    \mathcal{Q}_\mathcal{M}=\{[f:S \to \mathcal{M}] : S \subset \mathcal{M} \text{ and } (\mathcal{M},\text{Im}f) \text{ is pair-homotopy equivalent to } (\mathcal{M},\text{Im}g)\}.
\end{equation*}
Then $[W]_\mathcal{M}$ is generated by any sequence of unions of repetitions of elements of $\mathcal{Q}$.

\subsection{Representation classifying spaces of some manifolds}
As it has already been mentioned things become quite complex in high dimensional cases, for this reason here we try to build up some intuition by considering lower dimensional examples. 
The circle $S^1$ is one dimensional and therefore we know that all of its submanifolds are either intervals or the circle itself. In other words $\mathcal{Q}_{S^1} = \{[f: I \to S^1],[f: S^1 \to S^1]\}$. Quotiening out n disconnected intervals (or points, since both are contractible) is the same as identifying n distinct points together. By going through the long exact sequence of relative homology one can see that this generates the wedge sum of $n$ loops. To show this imagine that we have a neural network for which a rank 0 region intersects the circle in n disjoint segments, this implies that $H_0(\Omega_0)=n\mathbb{Z}$ and is trivial for higher groups. Then we can construct the long exact sequence,

\begin{equation}
    \begin{tikzcd}
        0 \arrow[r, "\partial^*"] & H_{1}(\Omega_0) \arrow[r, "\iota^*"] & H_1(S^1) \arrow[r, "j^*"] & H_1(S^1,\Omega_0) \arrow[r, "\partial^*"] & H_{0}(\Omega_0) \arrow[r, "\iota^*"] & H_0(S^1) \arrow[r, "j^*"] & 0.
    \end{tikzcd}
\end{equation}

We also know the homology groups of the circle, which are $H_0(S^1) = \mathbb{Z}$ and $H_1(S^1)=\mathbb{Z}$. Plugging all of this data into the sequence we get,

\begin{equation}
    \begin{tikzcd}
        0 \arrow[r, "\partial^*"] & 0 \arrow[r, "\iota^*"] & \mathbb{Z} \arrow[r, "j^*"] & H_1(S^1,\Omega_0) \arrow[r, "\partial^*"] & n\mathbb{Z} \arrow[r, "\iota^*"] & \mathbb{Z} \arrow[r, "j^*"] & 0.
    \end{tikzcd}
\end{equation}

The image of the first $j^*$ map is $\mathbb{Z}$, since we know that its kernel is empty by exactness $\ker j^*=\text{Im}\iota^*=0$. On the other side, we know that the image of the second $\iota^*$ is the same as the kernel of $j^*=\mathbb{Z}$, this then implies that its kernel is $(n-1)\mathbb{Z}$. Again, from exactness we know that this kernel is the same as the image of $\partial^*$. Therefore we know that $H_1(S^1,\Omega_0) = n\mathbb{Z}$. Alternatively $\Omega_0=S^1$, in which case $H_1(S^1,\Omega_0)$ is trivial. Therefore we can postulate that the \textit{representation classification} space of the circle is given by:

\begin{equation}
    [W]_{S^1} = \{(0,n\mathbb{Z},0,...)|n\geq0\}.
\end{equation}

A more interesting example is the torus, since it is higher dimensional and its fundamental group generates non-homotopic maps from the circle to the torus. Using this type of reasoning we can show that its \textit{representation classifying space} is given by quotiening out solid intervals and non-homotopic circles, or $\mathcal{Q}_{T^2} = \{[f: I^2 \to T^2], [p_0:S^1 \to T^2], [p_1:S^1 \to T^2], [p_2:S^1 \to T^2]\}$. Where $p_i$ projects to a loop on the torus which is generated by one of the three non-homotopic loops (contractible - 0, inner - 1 and outer - 2). Then after a more involved diagram chase the $[W]_{T^2}$ set is generated by,

\begin{equation}
    [W]_{T^2} = 
    \begin{cases}
        (0,(n+1)\mathbb{Z},\mathbb{Z},0,...) \text{ if n copies of } I^2,\\
        (0,(n+1)\mathbb{Z},(n+1)\mathbb{Z},0,...) \text{if n copies of trivial } S^1,\\
        (0,n\mathbb{Z},n\mathbb{Z},0,...) \text{ if n copies of outer xor inner loops } S^1,\\
    \end{cases}
\end{equation}

Combinations of these submanifolds can generate even more topologies. The effects of these combinations are listed in the following tables, where only $H_1(T^2,A_n\bigsqcup B_m)$ and $H_2(T^2,A_n\bigsqcup B_m)$ are shown, with $\bigsqcup$ denoting gluing together $n$ disjoint copies of $A$ and $m$ disjoint copies of $B$:

\begin{minipage}{0.5\textwidth}
\centering
\begin{tabular}{|c|c|c|c|}
\hline
$I_n \bigsqcup S^1_m$ & $m=1$ & $2$ & $3$ \\
\hline
$n=1$ & $(3\mathbb{Z},2\mathbb{Z})$ & $(4\mathbb{Z},3\mathbb{Z})$ & $(5\mathbb{Z},4\mathbb{Z})$ \\
\hline
$2$ & $(4\mathbb{Z},2\mathbb{Z})$ & $(5\mathbb{Z},3\mathbb{Z})$ & $(6\mathbb{Z},4\mathbb{Z})$ \\
\hline
$3$ & $(5\mathbb{Z},\mathbb{Z})$ & $(6\mathbb{Z},3\mathbb{Z})$ & $(7\mathbb{Z},4\mathbb{Z})$ \\
\hline
\end{tabular}
\end{minipage}
\begin{minipage}{0.5\textwidth}
\centering
\begin{tabular}{|c|c|c|c|}
\hline
$S^1_n \bigsqcup S^{1,1}_m$ & $m=1$ & $2$ & $3$ \\
\hline
$n=1$ & $(2\mathbb{Z},2\mathbb{Z})$ & $(3\mathbb{Z},3\mathbb{Z})$ & $(4\mathbb{Z},4\mathbb{Z})$ \\
\hline
$2$ & $(3\mathbb{Z},3\mathbb{Z})$ & $(4\mathbb{Z},4\mathbb{Z})$ & $(5\mathbb{Z},5\mathbb{Z})$ \\
\hline
$3$ & $(4\mathbb{Z},4\mathbb{Z})$ & $(5\mathbb{Z},5\mathbb{Z})$ & $(6\mathbb{Z},6\mathbb{Z})$ \\
\hline
\end{tabular}
\end{minipage}

\vspace{1cm}

\begin{minipage}{\textwidth}
\centering
\begin{tabular}{|c|c|c|c|}
\hline
$S^{1,1}_n \bigsqcup S^{1,2}_m$ & $m=1$ & $2$ & $3$ \\
\hline
$n=1$ & $(0,\mathbb{Z})$ & $(0,2\mathbb{Z})$ & $(0,3\mathbb{Z})$ \\
\hline
$2$ & $(0,2\mathbb{Z})$ & $(0,3\mathbb{Z})$ & $(0,4\mathbb{Z})$ \\
\hline
$3$ & $(0,3\mathbb{Z})$ & $(0,4\mathbb{Z})$ & $(0,5\mathbb{Z})$ \\
\hline
\end{tabular}
\end{minipage}

\vspace{1cm} 

Performing such descriptions for higher-dimensional cases is interesting, but becomes incredibly complex for higher-dimensional settings in which the number and complexity of the potential submanifolds blows up. Therefore, this problem might highly benefit from a more algorithmic treatment.

\begin{figure}
    \centering
    \includegraphics[scale=0.8]{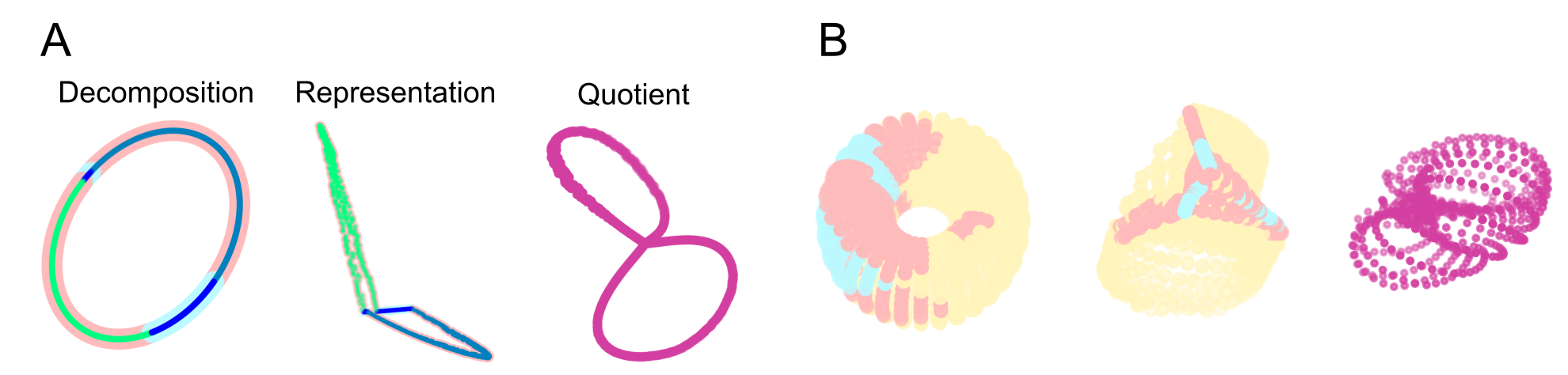}
    \caption{\textbf{A)} Polytope (inner points) and rank (outer regions) decomposition of a circle, induced by a neural network plotted on a circle (left) and on the hidden layer representation (middle). Topological representation of the manifold after low rank regions are identified (quotiened out). Multidimensional scaling was used to reduce the dimensionality in the last two plots. \textbf{B)} Same as \textbf{A)} but for a torus which is transformed with the outer loop compressed, only the rank decomposition is shown for visibility.}
    \label{fig:other-manifolds}
\end{figure}

\subsection{Topological destruction after training on other optimization functions}
\label{Other regressions}
To verify that our insights on the sphere function generalize to other settings, we trianed networks on several other optimization tasks. The two dimensional versions of the functions we use are shown in panel B of Figure \ref{fig:other regressions}. When training the network, we set the dimension $d$ to 20 and sampled 10000 uniformly distributed training and test samples from the interval the functions are defined over (this can be seen on http://www.sfu.ca/~ssurjano). Besides just fitting to these functions, we also created an alternative classification problem out of them in which, if $f(\textbf{x})\geq\hat{f}$ (with $\hat{f}$ being the mean value of a function) we set $f(\textbf{x})=1$. We set every other value to 0. 

These simulations seem to be consistent with the results in panel C of figure \ref{fig:training_ranks}, indicating that classification problems require the learning of a topologically destructive map, whereas regression problems rather lead to an increase in the average rank. The only exception was the Schwefel function for which the classification task lead to a strong reduction in rank, but the regression task did not exhibit an increase in rank.


\begin{figure}
    \centering
    \includegraphics[scale=0.8]{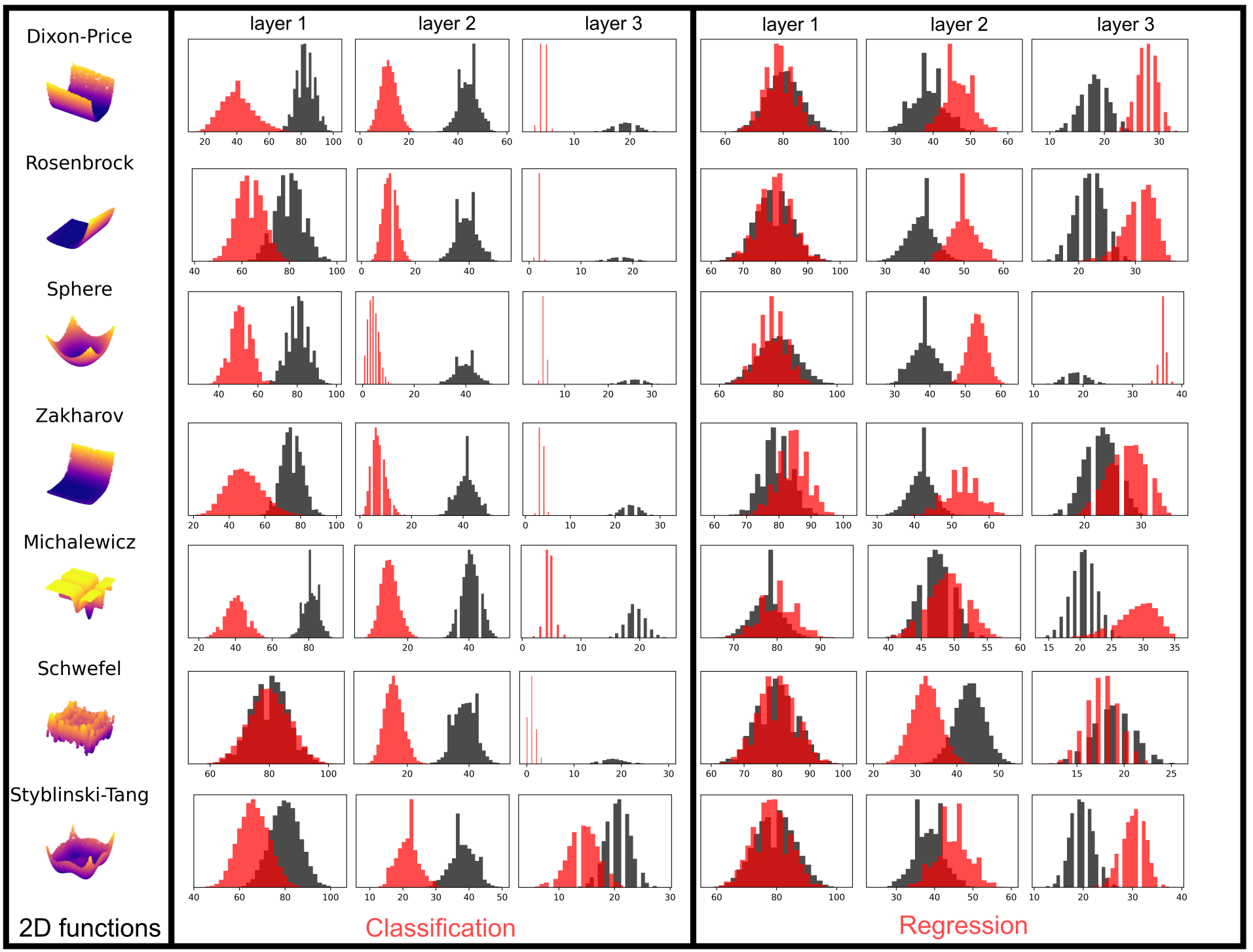}
    \caption{The first column shows pictures of the functions to which we fit in two dimensions (the actual fits are done on functions in 20 dimensions). In the classification problem case, we observe that after training (red) the samples of the test data fall in regions of codeword size much lower than before training (black). When it comes to regression, the pattern inverses with data samples falling in regions of lower codeword size before training compared to after training.}
    \label{fig:other regressions}
\end{figure}

\end{document}